\documentclass[a4paper,10pt]{article}

\usepackage[margin=1in]{geometry}
\usepackage{microtype}
\usepackage{makecell}
\usepackage{booktabs}

\usepackage{amsmath,amssymb,amsfonts,mathrsfs}
\usepackage{amsthm}

\usepackage{graphicx}
\usepackage{subcaption}
\usepackage{booktabs}
\usepackage{threeparttable}
\usepackage{multirow}
\usepackage{tikz}

\usepackage{enumitem}
\usepackage[unicode,colorlinks,linkcolor=blue,citecolor=blue,urlcolor=blue]{hyperref}

\newcommand{\includegraphicsmaybe}[2][]{%
  \IfFileExists{#2}{\includegraphics[#1]{#2}}{%
    \fbox{\parbox{0.9\linewidth}{\centering \texttt{Missing file: #2}}}%
  }%
}

\newtheorem{theorem}{Theorem}[section]
\newtheorem{lemma}[theorem]{Lemma}

\title{\bf Randomized neural operator for parametric PDEs with fast training and conformal uncertainty quantification}
\author{
Zirui Deng$^{1}$ \quad Jingbo Sun$^{1}$ \quad Deyu Meng$^{1}$ \quad and \quad Fei Wang$^{1,*}$\\[2pt]
{\small $^{1}$School of Mathematics and Statistics, Xi'an Jiaotong University, Xi'an, Shaanxi 710049, China}\\
{\small $^\ast$Correspondence: \texttt{feiwang.xjtu@xjtu.edu.cn}}
}
\date{}

\begin{document}
\maketitle

\begin{abstract}

Repeatedly solving parametric PDEs is essential for uncertainty quantification, design optimization and inverse problems, but conventional neural operators require expensive non-convex training. We introduce PCA--RaNN, a randomized latent neural operator that combines PCA-based dimensionality reduction with fixed random features and a closed-form least-squares readout. It recasts latent operator learning as fixed-feature linear regression, reducing training time by one to three orders of magnitude across benchmarks while maintaining competitive accuracy. We introduce an energy-matched scaling rule and a lightweight two-parameter BFGS refinement to correct suboptimal feature scales. Ensemble averaging reduces predictive variance. On Burgers, Darcy, Navier--Stokes and backward heat equation benchmarks, PCA--RaNN provides a favorable speed--accuracy trade-off against operator-learning baselines. The ensemble supports split-conformal prediction intervals, and the linear readout enables rapid online adaptation via recursive least squares without retraining hidden features. This provides an efficient, uncertainty-aware surrogate for many-query scientific workflows.

\end{abstract}

\vspace{0.3em}
\noindent\textbf{Keywords:} operator learning; reduced-order modeling; randomized neural networks; uncertainty quantification; parametric partial differential equations.

\vspace{0.5em}

The modeling and simulation of physical systems governed by partial differential equations (PDEs) constitute a cornerstone of modern scientific computing, with applications spanning fluid dynamics, solid mechanics and electromagnetics~(\cite{Ferziger2002Computational, Zienkiewicz2005Finite, Jin2015Finite}). Many such problems are parametric, requiring repeated solutions for varying coefficients, source terms, initial conditions or boundary data, as arise in design optimization, uncertainty quantification and inverse problems. Traditional numerical methods, including finite element, finite difference and spectral methods, provide reliable approximations but can be computationally expensive, especially in high-resolution or many-query settings. Since each new parameter instance typically requires solving a large-scale discretized system, real-time prediction and extensive parameter exploration remain challenging, motivating the development of efficient surrogate models.

Deep learning has recently introduced a powerful paradigm for data-driven surrogate modeling. Neural networks, known for their universal approximation capabilities~(\cite{Hornik1989Multilayer}), have been extended from function approximation to the learning of mappings between infinite-dimensional function spaces, namely operators. Foundational theoretical results established that neural networks can approximate nonlinear operators~(\cite{Chen1995Universal,Chen1995Approximation}), leading to practical neural operator architectures such as Deep Operator Networks (DeepONets, \cite{Lu2021Learning}) and Fourier Neural Operators (FNOs, \cite{Li2020Fourier,Li2020Neural}). DeepONet uses a branch--trunk architecture to encode input functions and evaluate outputs at query locations, whereas FNO learns integral operators efficiently in the Fourier domain. 
Physics-informed variants further incorporate governing equation residuals into the training objective, reducing the demand for labeled data~(\cite{Li2024Physics,Wang2021Learning}). 
Recent reviews have further highlighted neural operators as a promising framework for accelerating scientific simulations and design, and have identified operator learning, reduced-order representations and machine-learning-assisted numerical algorithms as central directions in machine learning for PDEs~(\cite{Azizzadenesheli2024Neural,Brunton2024Promising}). 
Despite these advances, many neural-operator models still require training deep networks with many parameters through non-convex optimization, which can be computationally demanding and sensitive to initialization, hyperparameters and training protocols.

This computational burden has motivated the integration of operator learning with model reduction techniques, which exploit the low-dimensional structure often present in PDE solution manifolds. Classical reduced basis methods~(\cite{DeVore2017Theoretical}) can construct efficient surrogates but are usually intrusive, requiring access to the governing equations or weak forms. Data-driven non-intrusive alternatives have therefore attracted increasing attention. For example, PCA-based neural networks (PCA--NNs, \cite{Bhattacharya2021Model}) use principal component analysis (PCA) to compress input and output fields into low-dimensional latent coordinates and then learn the latent mapping using a fully connected neural network. Autoencoder-based latent neural operators, such as L-DeepONet~(\cite{Kontolati2024Learning}), similarly learn nonlinear embeddings and approximate the reduced operator in latent space, enabling real-time prediction of complex high-dimensional dynamics. Although these approaches demonstrate the effectiveness of latent operator learning, the latent-space regressor is still typically trained by gradient-based optimization and therefore remains subject to training cost and convergence issues.

Randomized neural networks (RaNNs) provide an alternative route to fast training~(\cite{Igelnik1995Stochastic,Igelnik1999Ensemble,Pao1994Learning}). In RaNNs, the weights and biases of the hidden layers are randomly sampled and fixed, so that the network output becomes a linear combination of random nonlinear features. Only the output-layer weights are trained, usually by solving a least-squares problem. This converts non-convex neural-network training into linear regression, reducing training time and avoiding difficulties such as local minima and vanishing gradients. RaNNs have shown strong performance in forward and inverse PDE problems~(\cite{Chen2022Bridging,Dong2021Local,Shang2023Randomized,Sun2024Local,Sun2024Diffusive,Li2025Local}) and have recently been incorporated into operator-learning frameworks~(\cite{Jiang2025DeepONet}). However, standard randomized architectures still require appropriate choices of random-feature scales, and single-hidden-layer constructions may lack the flexibility needed to represent complex multiscale operator maps. 
 Adaptation under shifted regimes is another important issue: transfer operator learning has been explored for PDE problems by adapting DeepONet-based models from source tasks to related target tasks under conditional shift~(\cite{Goswami2022Deep}). 
Taken together, these observations leave three practical gaps. First, latent operator learning still often relies on non-convex optimization even after dimensionality reduction. Second, randomized-feature methods are fast but remain sensitive to feature-scale selection, especially in latent operator-learning settings. Third, many-query scientific workflows require not only fast offline training, but also uncertainty estimates and rapid adaptation to shifted regimes.

To address these gaps, we introduce PCA--RaNN, a randomized latent operator-learning framework that combines PCA-based dimensionality reduction with a two-hidden-layer RaNN equipped with a hidden-layer skip connection. After projecting input and output fields into latent spaces, PCA--RaNN replaces the trainable nonlinear latent regressor used in PCA–NN-type reduced models by a fixed random feature map with a closed-form least-squares readout. 
This is the main distinction from conventional PCA-based neural surrogates: the latent coordinates are still obtained from data, but the latent solution map is fitted as a fixed-feature linear regression problem rather than trained by backpropagation. Rather than treating random-feature scales as manually tuned hyperparameters, we introduce a layer-wise energy-matching rule that provides a data-driven default initialization without a costly grid search. Because this energy-based criterion does not directly minimize prediction error, 
we further introduce a lightweight BFGS refinement over only two logarithmic scale parameters. For each fixed pair of scale parameters, the readout is still obtained by least squares, while the scale parameters are selected by validation. We also use ensemble averaging to reduce prediction variance. The ensemble structure enables uncertainty quantification through split conformal prediction~(\cite{Angelopoulos2021Gentle,Papadopoulos2002Inductive,Vovk2005Algorithmic}), yielding distribution-free pooled marginal prediction intervals over samples and grid points under the standard exchangeability assumption. In addition, the linear readout allows rapid online adaptation through recursive least-squares updates~(\cite{Haykin2008Adaptive,Sayed2003Fundamentals}), enabling the model to adapt to new data or shifted regimes without retraining the hidden features.

We evaluate PCA--RaNN on four operator-learning benchmarks: the viscous Burgers' equation, Darcy flow, the incompressible Navier--Stokes equations and the backward heat equation. Across these tasks, PCA--RaNN substantially reduces training time relative to operator-learning baselines while achieving competitive, and in several cases superior, accuracy. The BFGS-refined variant is particularly effective when the energy-matched scales are suboptimal, as observed in the backward heat equation benchmark. These results suggest that randomized latent neural operators can serve as lightweight, uncertainty-aware and rapidly adaptable surrogates for PDE-governed systems, especially in settings where training time, computational resources or repeated model updates are limiting factors.

\section*{Results}

\subsection*{PCA--RaNN achieves efficient operator learning via latent-space random features}

Learning solution operators for parametric PDEs directly in high-dimensional physical space is computationally prohibitive. PCA--RaNN circumvents this bottleneck by operating in a compressed latent space: input and output fields are encoded by PCA into low-dimensional coordinates $x$ and $y$, and a two‑hidden‑layer randomized neural network with a skip connection (with fixed hidden weights) maps
\[
    \widehat{y} = \Phi(x) \approx y, \qquad x=F_{\mathcal{A}}(a),\quad y=F_{\mathcal{U}}(u).
\]
Crucially, only the final readout weights are fitted by least squares, making the latent operator learning a linear regression problem. This is the main distinction from PCA--NN-type reduced models: the latent coordinates are data-driven, but the nonlinear latent regressor is replaced by a fixed random feature map with a closed-form readout. The physical prediction is reconstructed as
\[
    \widehat{u} = G_{\mathcal{U}}\left(\Phi(F_{\mathcal{A}}(a))\right).
\]
Figure~\ref{fig:workflow} illustrates the overall workflow and the latent regressor architecture.

We study two variants of the same randomized latent operator. The first, the energy-matched variant, determines the feature scales $(\bar r_1,\bar r_2)$ by the layer-wise energy-matching rule described in Methods. This requires only PCA, random-feature construction, and a single least-squares solve. The second, the BFGS-refined variant, starts from the same energy-matched scales and refines them by minimizing a validation loss over the two logarithmic scale parameters. This refinement involves only a two-dimensional validation optimization; for each fixed pair of scales, the output-layer fit remains a closed-form least-squares or ridge-regression solve.

\begin{figure}[t]
    \centering
        \includegraphics[width=\linewidth]{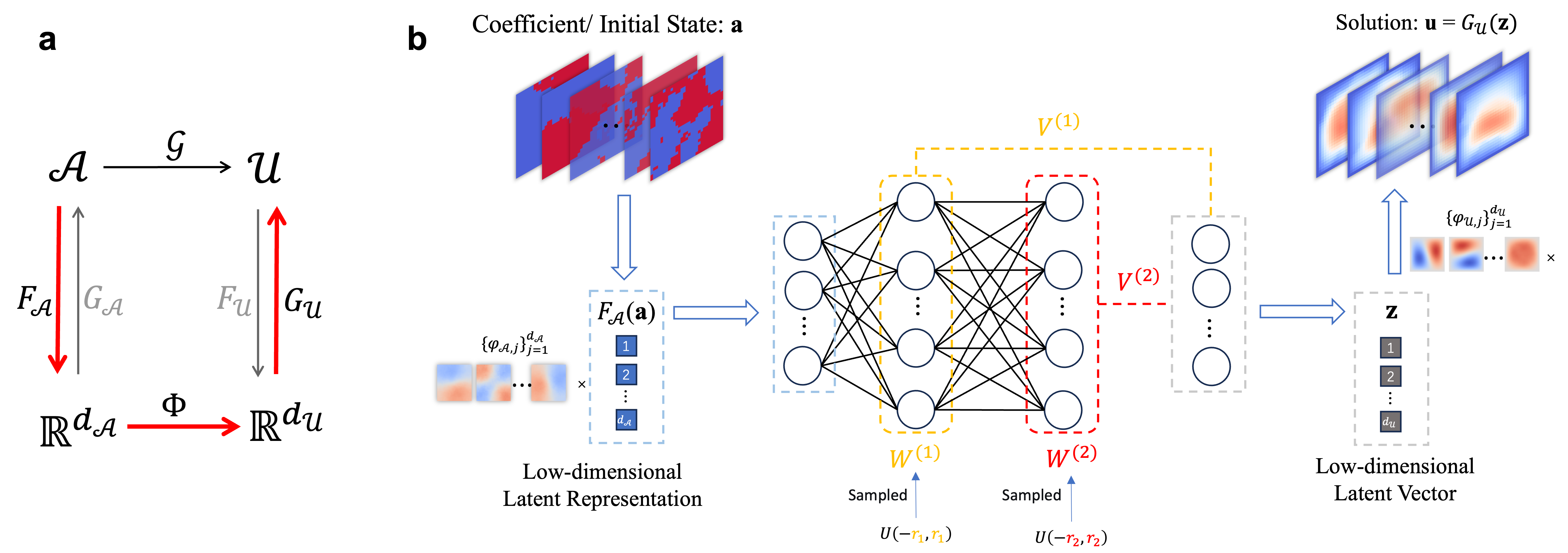}
    \caption{\textbf{Overview of PCA--RaNN.}
   \textbf{a,}  Latent-space decomposition of the parametric PDE solution operator.
    The input field \(a\in\mathcal A\) is encoded into a low-dimensional PCA coordinate
    \(x=F_{\mathcal A}(a)\), mapped by a latent operator \(\Phi\), and decoded to the
    physical prediction \(\hat u=G_{\mathcal U}(\Phi(x))\).
   \textbf{b,}  Architecture of the randomized latent regressor.
    The hidden weights \(W^{(1)}\) and \(W^{(2)}\) are randomly sampled and fixed,
    with feature scales \(r_1\) and \(r_2\) selected by energy matching and optionally
    refined by BFGS. The readout weights \(V^{(1)}\) and \(V^{(2)}\) are fitted by
    least squares, with a skip connection from the first hidden layer to the output.}
    \label{fig:workflow}
\end{figure}


\subsection*{BFGS refinement corrects suboptimal energy-matched scales}

The scales $r_1$ and $r_2$ control the magnitude of the random projections before the nonlinear activation. Rather than treating these scales as manually tuned hyperparameters, PCA--RaNN first matches the expected hidden-feature energy to the output latent energy. The energy-matching rule provides data-driven initial values without a grid search, but it is based on a small-signal approximation and therefore does not directly minimize the validation error. To improve upon this initialization, we introduce two logarithmic scale-offset parameters, $\eta_1$ and $\eta_2$, which multiplicatively adjust the energy-matched scales $\bar r_1$ and $\bar r_2$ as
\[
    r_1(\eta_1)=\bar r_1\exp(\eta_1),
    \qquad
    r_2(\eta_2)=\bar r_2\exp(\eta_2).
\]
We then select \(\eta=(\eta_1,\eta_2)^\top\) by approximately minimizing a held-out validation loss with BFGS.

Table~\ref{tab:scale_refinement} reports the energy-matched scales, the refined scales, and the number of BFGS iterations. Across the benchmark problems, BFGS either adjusts the energy-matched scales when the validation loss can be reduced or leaves them unchanged when the initial scales already satisfy the stopping criterion. The improvement is most pronounced for the backward heat equation benchmark, where the refined ensemble model reduces the relative test error from $2.89\times 10^{-1}$ to $5.59\times 10^{-2}$. The refinement also improves Burgers' equation and Navier--Stokes, reducing the errors from $2.28\times 10^{-3}$ to $1.93\times 10^{-3}$ and from $2.59\times 10^{-2}$ to $2.47\times 10^{-2}$, respectively. For Darcy flow, BFGS terminates at the initial point and does not identify a validation-loss decrease from the energy-matched initialization; consequently, the refined scales coincide with the initial ones. Together, these results indicate that energy matching provides a useful default initialization and that a two-dimensional refinement can recover additional accuracy when the initial scales are suboptimal.

\begin{table*}[t]
\centering
\caption{\textbf{Energy-matched and BFGS-refined scales with test errors.}
The energy-matched scales $(\bar r_1,\bar r_2)$ are computed from the training data. The BFGS-refined scales are $r_l^\star=\bar r_l\exp(\eta_l^\star)$. Error (EM) and Error (BFGS) denote test relative $L^2$ errors of the ensemble model. Improvement is the percentage reduction of Error (BFGS) relative to Error (EM), computed from the unrounded errors before formatting. For Darcy flow, BFGS stops at the initial scales; the small EM--BFGS error difference arises from the PCA basis used in the BFGS validation protocol rather than from scale refinement.}
\label{tab:scale_refinement}
\resizebox{\textwidth}{!}{ 
\begin{tabular}{lccccccc}
\toprule
Problem & $(\bar r_1, \bar r_2)$ & $(r_1^\star, r_2^\star)$ & $\eta^\star$ & BFGS iters & Error (EM) & Error (BFGS) & Improvement \\
\midrule
Burgers & (0.45, 0.10) & (0.54, 0.10) & (0.18, 0.00) & 17 & $2.28\times 10^{-3}$ & $1.93\times 10^{-3}$ & 15.58\% \\
Darcy & (0.00046, 0.03) & (0.00046, 0.03) & (0.00, 0.00) & 0 & $2.95\times 10^{-2}$ & $2.94\times 10^{-2}$ & 0.21\% \\
Navier--Stokes & (0.024, 0.061) & (0.024, 0.055) & (0.01, -0.11) & 9 & $2.59\times 10^{-2}$ & $2.47\times 10^{-2}$ & 4.51\% \\
Backward heat eqn. & (0.48, 0.24) & (0.0083, 0.34) & (-4.04, 0.34) & 5 & $2.89\times 10^{-1}$ & $5.59\times 10^{-2}$ & 79.46\% \\
\bottomrule
\end{tabular}
}%
\end{table*}


\subsection*{Benchmark performance and training efficiency}

We evaluate PCA--RaNN on four operator-learning tasks: Burgers' equation (initial condition $\to$ full spatio-temporal solution), Darcy flow (permeability field $\to$ pressure field), the two-dimensional Navier--Stokes equation in vorticity form (10 earlier vorticity snapshots $\to$ 10 later ones), and the backward heat equation (final temperature field $\to$ previous 20 time steps). Detailed problem descriptions, dataset generation and model configurations are provided in Methods.

BFGS refinement improves accuracy to different extents across the benchmarks, at the cost of additional validation-loss evaluations over the two scale parameters. We therefore report both ensemble variants ($K=20$) to show the accuracy--training cost trade-off. Table~\ref{tab:benchmark} compares PCA--RaNN with three representative operator-learning baselines, PCA--NN~(\cite{Bhattacharya2021Model}), DeepONet~(\cite{Lu2021Learning,Lu2022Comprehensive}) and FNO~(\cite{Li2020Fourier}). The architectures and training budgets of the reimplemented baselines are summarized in Supplementary Table~\ref{tab:baseline_config}. All training times are measured on a single NVIDIA RTX A6000 GPU.

On Burgers' equation, PCA--RaNN is both faster and more accurate than the reimplemented baselines. 
The BFGS-refined ensemble trains in \(5.5\,\mathrm{s}\) and achieves a relative error of \(1.93\times10^{-3}\), compared with \(2.28\times10^{-2}\), \(1.11\times10^{-2}\) and \(9.87\times10^{-3}\) for DeepONet, FNO and PCA--NN, respectively. 
On Darcy flow, FNO gives the lowest error, whereas PCA--RaNN remains competitive and is substantially faster: the BFGS-refined ensemble attains \(2.94\times10^{-2}\) in \(14.6\,\mathrm{s}\), compared with \(2.43\times10^{-2}\) in \(1258.6\,\mathrm{s}\) for FNO.

For Navier--Stokes, FNO attains the lowest final-time error ($2.03\times10^{-2}$), followed by DeepONet ($2.23\times10^{-2}$) and the BFGS-refined PCA--RaNN ensemble ($2.47\times10^{-2}$). PCA--RaNN is less accurate than FNO and DeepONet on this benchmark, but trains in $88.7\,\mathrm{s}$, compared with $1150.3\,\mathrm{s}$ for FNO, $2000.7\,\mathrm{s}$ for DeepONet and $1468.2\,\mathrm{s}$ for PCA--NN.

For the backward heat equation benchmark, the EM model is less accurate than PCA--NN ($2.89\times10^{-1}$ vs. $1.17\times10^{-1}$). After BFGS refinement, however, the error drops to $5.59\times10^{-2}$, outperforming DeepONet, FNO and PCA--NN while reducing training time from $2221.7\,\mathrm{s}$ for FNO and $1507.1\,\mathrm{s}$ for PCA--NN to $9.8\,\mathrm{s}$. This sensitivity is expected because the backward heat equation is ill-posed, so changes in the random-feature scale parameters can strongly affect the recovery of earlier-time components. Overall, the results indicate that PCA--RaNN provides a favorable speed--accuracy trade-off: it is not always the most accurate model, but it consistently reduces training time by one to three orders of magnitude.

Ensemble averaging reduces the test error by 10--30\% relative to a single PCA--RaNN member, with gains exceeding 50\% for Burgers' equation (Extended Data Table~\ref{tab:ensemble_vs_single}). The training cost scales approximately linearly with the ensemble size because each additional member only requires random-feature construction and a least-squares readout fit. Representative predictions for each benchmark are shown in Extended Data Figs.~\ref{fig:extended_Burgers}--\ref{fig:extended_Heat}. 
Taken together, these results suggest that PCA--RaNN is most attractive when retraining cost, repeated model updates or limited computational budgets are the dominant bottlenecks, whereas fully trained neural operators may remain preferable when the main objective is to minimize the prediction error on a fixed training distribution.

\begin{table*}[t]
\centering
\caption{\textbf{Performance comparison on PDE benchmarks.}
Training time is wall-clock time measured on a single NVIDIA RTX A6000 GPU. For PCA--RaNN, both the energy-matched (EM) and BFGS-refined ensembles ($K=20$) are reported. DeepONet, FNO and PCA--NN are reimplemented and evaluated under the same benchmark settings. For Navier--Stokes, all methods predict the sequence $t=11$--$20$, and the reported error is the mean relative $L^2$ error at the final predicted time $t=20$.}
\label{tab:benchmark}
\begin{tabular}{@{}llcc@{}}
\toprule
Problem & Method & Training time (s) & Rel. $L^2$ error \\
\midrule
\multirow{5}{*}{Burgers}
 & DeepONet & 32330.9 & $2.28\times10^{-2}$ \\
 & FNO & 13468.7 & $1.11\times10^{-2}$\\
 & PCA--NN & 2878.8 & $9.87\times10^{-3}$ \\
 & PCA--RaNN-EM (ensemble) & 2.3 & $2.28\times10^{-3}$\\
 & PCA--RaNN-BFGS (ensemble) & 5.5 &  $1.93\times10^{-3}$\\
\midrule
\multirow{5}{*}{Darcy}
 & DeepONet & 1740.1 & $3.01\times10^{-2}$ \\
 & FNO & 1258.6 & $2.43\times10^{-2}$\\
 & PCA--NN & 1546.3 & $3.35\times10^{-2}$ \\
 & PCA--RaNN-EM (ensemble) & 10.6 & $2.95\times10^{-2}$ \\
 & PCA--RaNN-BFGS (ensemble) & 14.6 & $2.94\times10^{-2}$ \\
\midrule
\multirow{5}{*}{Navier--Stokes}
 & DeepONet & 2000.7 & $2.23\times10^{-2}$ \\
 & FNO & 1150.3 & $2.03\times10^{-2}$\\
 & PCA--NN & 1468.2 & $3.34\times10^{-2}$ \\
 & PCA--RaNN-EM (ensemble) & 19.4 & $2.59\times10^{-2}$ \\
 & PCA--RaNN-BFGS (ensemble) & 88.7 & $2.47\times10^{-2}$ \\
\midrule
\multirow{5}{*}{Backward heat eqn.}
 & DeepONet & 2736.9 & $2.53\times10^{-1}$ \\
 & FNO & 2221.7 & $3.21\times10^{-1}$\\
 & PCA--NN & 1507.1 & $1.17\times 10^{-1}$ \\
 & PCA--RaNN-EM (ensemble) & 4.6 & $2.89\times 10^{-1}$ \\
 & PCA--RaNN-BFGS (ensemble) & 9.8 & $5.59\times10^{-2}$ \\
\bottomrule
\end{tabular}
\end{table*}


\subsection*{Uncertainty quantification with conformal prediction}

We equip the ensemble PCA--RaNN with prediction intervals using split conformal prediction (Methods). For each benchmark, a calibration set of size $N_{\rm cal}=300$ is used to compute nonconformity scores and the conformal quantile corresponding to a nominal coverage level $1-\alpha=0.95$. The calibration set is disjoint from the training, BFGS-validation and test sets, and is not used for model fitting or selection; the predictor is frozen before calibration. The resulting prediction intervals are then evaluated on independent test sets. Since the uncertainty procedure depends only on the ensemble predictions and not on how the scales were selected, we report it for the BFGS-refined ensemble model; the same procedure applies directly to the energy-matched ensemble.

Figure~\ref{fig:UQ_results}\subref{fig:UQ_results:a} presents the reliability diagrams for all four benchmarks: Burgers' equation, Darcy flow, Navier--Stokes and the backward heat equation. For each problem, the empirical coverage curve, computed by pooling across test samples and grid points, closely follows the diagonal across the tested confidence levels, consistent with the finite-sample pooled marginal coverage guarantee of split conformal prediction under exchangeability. At the nominal 95\% level, the empirical coverages are $94.71\%$, $95.52\%$, $95.50\%$ and $94.69\%$ for Burgers, Darcy, Navier–Stokes and the backward heat equation, respectively. The corresponding grid-averaged interval widths are $1.29\times 10^{-3}$, $8.87\times 10^{-4}$, $6.58\times 10^{-2}$ and $1.92\times 10^{-2}$, respectively.

A representative Darcy test sample (the one with median prediction error) is examined in detail in Fig.~\ref{fig:UQ_results}\subref{fig:UQ_results:b}–\subref{fig:UQ_results:e}. The true pressure field (Fig.~\ref{fig:UQ_results}\subref{fig:UQ_results:b}) exhibits sharp transitions induced by the heterogeneous permeability field. The ensemble mean prediction (Fig.~\ref{fig:UQ_results}\subref{fig:UQ_results:c}) captures the overall structure accurately. A cross-sectional slice at $y=0.5$ (Fig.~\ref{fig:UQ_results}\subref{fig:UQ_results:d}) shows that the 95\% prediction interval contains the true profile at all sampled points along this slice. This is an empirical observation for a single test sample and does not imply simultaneous coverage. The predictive standard deviation field (Fig.~\ref{fig:UQ_results}\subref{fig:UQ_results:e}) highlights regions of elevated variability that coincide with high-contrast permeability features, suggesting spatially varying predictive reliability. Together with the conformal calibration, this yields prediction intervals with near-nominal empirical coverage and spatial patterns that are interpretable in relation to the underlying permeability field.

\begin{figure}[t]
    \centering
    \includegraphics[width=\textwidth]{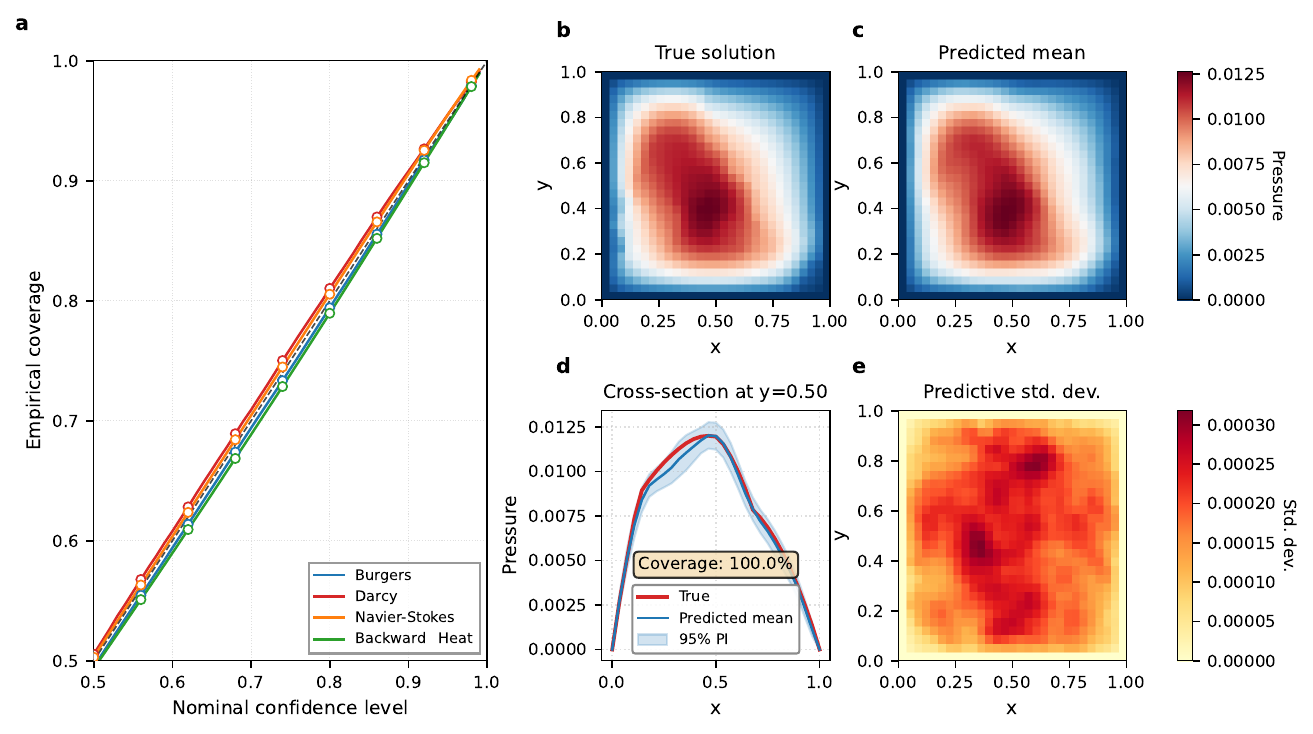}
    \caption{\textbf{Uncertainty quantification for ensemble PCA--RaNN across four PDE problems.}
    \textbf{a}, Reliability diagram showing empirical coverage versus nominal confidence level for Burgers (blue), Darcy (red), Navier–Stokes (orange), and the backward heat equation (green). The dashed diagonal indicates perfect calibration. 
    \textbf{b}--\textbf{e}, Detailed analysis of a representative Darcy test sample (median prediction error). \textbf{b}, True pressure solution. \textbf{c}, Ensemble mean prediction. \textbf{d}, Cross-section at $y=0.5$ comparing true (red) and predicted (blue) pressures; the shaded region denotes the 95\% prediction interval. \textbf{e}, Predictive standard deviation, with warmer colors indicating higher uncertainty.}
    \label{fig:UQ_results}
    \phantomsubcaption\label{fig:UQ_results:a}
    \phantomsubcaption\label{fig:UQ_results:b}
    \phantomsubcaption\label{fig:UQ_results:c}
    \phantomsubcaption\label{fig:UQ_results:d}
    \phantomsubcaption\label{fig:UQ_results:e}
\end{figure}


\subsection*{Rapid online adaptation via recursive least squares}

Finally, we demonstrate rapid online adaptation enabled by the linear readout. Once the PCA representation, randomized hidden features, and BFGS-refined scale parameters are fixed, the model can adapt to new data by updating only the linear readout. This adaptation mechanism does not require retraining the hidden layers. The adaptation samples are drawn from a later time window and are disjoint from the test set; all reported errors are evaluated on a separate held-out test set from that same window.

We test this capability on the long-horizon Navier--Stokes task. The model is first trained on 1000 samples mapping vorticity snapshots from $t=1$--$10$ to $t=11$--$20$. It is then adapted using a limited number of new samples from a later time window, mapping $t=11$--$20$ to $t=21$--$30$, with the adaptation set being much smaller than the original training set. During adaptation, the PCA basis, randomized hidden layers, anchor values, and BFGS-refined scale parameters are kept fixed; only the readout weights are updated via recursive least squares (RLS). As a fixed-feature baseline, we compare RLS with batch readout refitting, where the linear readout is recomputed by solving a least-squares problem on the augmented dataset (combining old and new samples) while keeping the PCA basis and hidden features fixed. In contrast, RLS incrementally updates the readout and inverse covariance from the original training data using only the new adaptation samples, without revisiting the original feature matrix.

Figure~\ref{fig:rls}\subref{fig:rls:a} shows the relative $L^2$ error at the farthest prediction time ($t=30$) as a function of the number of adaptation samples. The original model without adaptation (dashed horizontal line) yields an error of $0.290$. Using $200$ adaptation samples, the error drops to $0.079$, corresponding to a $72.8\%$ reduction relative to the unadapted model. A similar improvement is observed for the average error over $t=21$--$30$, which decreases by $78.3\%$ (from $0.228$ to $0.049$).

Figure~\ref{fig:rls}\subref{fig:rls:b} compares the online adaptation time. The reported time includes feature evaluation on new samples and the update or refit of the linear readout. The hidden features of the original training samples are cached after the initial training stage and are reused by the batch readout refit. For $200$ samples, RLS completes the update in $0.60$ seconds, whereas batch readout refitting takes $4.5$ seconds, corresponding to a $7.5\times$ speedup.

These results highlight a practical advantage of PCA--RaNN: once the randomized latent feature map has been constructed, the model can be rapidly adapted to new data or shifted temporal regimes by updating only the linear readout. This is particularly useful in many-query simulations where new samples arrive sequentially or the operating
regime shifts, because the model can be updated without reconstructing PCA bases or retraining hidden
features.

\begin{figure}[t]
    \centering
    \includegraphics[width=\linewidth]{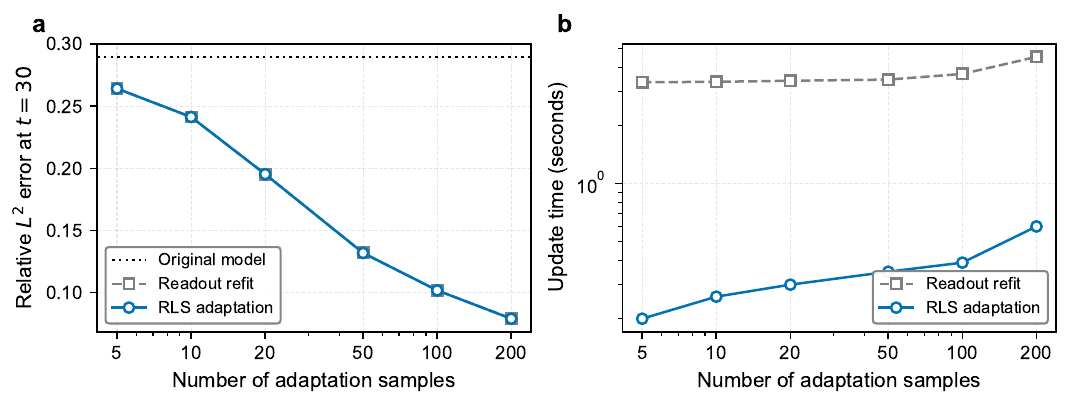}
    \caption{
    \textbf{Rapid online adaptation for Navier--Stokes.}
    \textbf{a,} Relative $L^2$ error at the farthest prediction time ($t=30$) versus the number of adaptation samples. The dashed horizontal line denotes the error of the original model without adaptation.
    \textbf{b,} Online adaptation time of RLS compared with batch readout refitting. Both methods use the same fixed PCA--RaNN features; batch readout refitting recomputes the linear readout by solving an augmented least-squares problem.
    }
    \label{fig:rls}
    \phantomsubcaption\label{fig:rls:a}
    \phantomsubcaption\label{fig:rls:b}
\end{figure}


\section*{Discussion}

We have introduced PCA--RaNN, a randomized latent neural operator that turns reduced operator learning into a fixed-feature linear regression problem. The main idea is to separate field compression, latent-space regression and physical reconstruction. PCA compresses the input and output fields into low-dimensional coordinates, a two-hidden-layer randomized neural network with a skip connection defines a nonlinear latent feature space, and only the final readout is fitted by least squares. Compared with PCA--NN-type reduced models, the key difference is that the latent map is not represented by a fully trained neural network, but by a randomized feature model with a closed-form readout for fixed feature scales. This design avoids backpropagation through all network parameters and reduces training time by approximately one to three orders of magnitude relative to the reimplemented baselines in the tested benchmarks, with the largest gains observed on Burgers' equation and the backward heat equation.

A central component of the method is scale selection for the random features. Randomized-feature models are computationally attractive, but their accuracy can be sensitive to the magnitude of the random projections before activation. The proposed layer-wise energy-matching rule provides a data-driven initialization by matching the expected hidden-feature energy to the output latent energy, thereby avoiding a costly grid search over random-feature scales. Because this rule is based on an approximate small-signal argument, the resulting scales can still be suboptimal. The two-parameter BFGS refinement addresses this issue by optimizing only two logarithmic scale parameters on a validation loss, while the readout remains a least-squares or ridge-regression solve for each fixed scale pair. This refinement is especially important for the backward heat equation benchmark, where the ill-posed inverse map makes the recovery more sensitive to feature-scale selection. For Darcy flow and Navier--Stokes, the smaller gains indicate that the energy-matched initialization is already close to adequate.

The linear-readout structure also makes PCA--RaNN useful beyond fast offline training. Ensemble averaging reduces the variance induced by random features, and split conformal prediction calibrates the ensemble outputs into prediction intervals with finite-sample pooled marginal coverage over sample--grid-point pairs under exchangeability. Empirically, the calibrated intervals achieve near-nominal coverage across the tested benchmarks, while the ensemble standard deviations reveal input-dependent spatial uncertainty patterns. The same linear structure also enables rapid online adaptation: once the PCA bases, hidden features and scale parameters are fixed, new data can be incorporated by recursive least-squares updates of the readout. The Navier--Stokes long-horizon experiment demonstrates this advantage, showing a substantial error reduction under a shifted temporal regime without retraining the hidden features.

Several limitations should be noted. First, the PCA encoder and decoder are tied to the training discretization, so the current implementation does not provide mesh-independent evaluation or continuous interpolation like coordinate-based decoders. This limitation is particularly relevant for geometry-dependent PDE problems, where diffeomorphic mapping operator learning has recently been developed to learn solution operators over varying geometries~(\cite{Yin2024Scalable}). Second, physics-informed training is nontrivial because the PDE residual is not directly represented in the latent coordinates. Third, while computationally efficient, the randomized feature map can be less accurate than fully trained neural operators on highly nonlinear or long-horizon turbulent dynamics, as seen in the Navier--Stokes benchmark. Finally, the approximation result in Supplementary Note~1 establishes high-probability universality for a single-hidden-layer RaNN and therefore supports the randomized-feature principle, but it does not yet characterize the two-hidden-layer skip architecture, PCA-compressed operator learning or BFGS-refined scale selection used here.

These limitations also highlight that PCA--RaNN should be viewed as a modular latent operator-learning framework rather than a fixed architecture. The present implementation uses PCA for compression, a two-hidden-layer RaNN for latent regression, energy matching and BFGS refinement for scale selection, ensemble conformal prediction for uncertainty calibration, and recursive least squares for online adaptation. Each component can be replaced or strengthened. The PCA encoder and decoder may be generalized to nonlinear, geometry-aware or physics-informed latent representations; the randomized latent regressor may be enriched with multiscale, local, structured or graph-based random features; and the scale-selection and calibration modules may be made adaptive to residual information or distribution shifts. 
Another complementary direction is to couple neural operators more tightly with classical numerical methods; for example, hybrid neural–numerical methods have blended DeepONet with relaxation schemes to balance convergence across eigenmodes of the iterative PDE solver~(\cite{Zhang2024Blending}). 
Provided that the latent representation and hidden features are fixed before readout fitting, the regression step can still retain the main computational advantage of PCA--RaNN, namely fast least-squares training and efficient online adaptation.

Overall, PCA--RaNN provides a lightweight and flexible foundation for fast, uncertainty-aware and adaptable neural operators for many-query PDE problems. It is not intended to replace fully trained neural operators in all regimes; rather, it offers a complementary route in settings where rapid training, calibrated uncertainty and efficient redeployment under limited computational budgets are essential.

\section*{Methods}

\subsection*{Operator learning formulation}

Let $\mathcal{A}$ and $\mathcal{U}$ be Hilbert spaces representing the input and output function spaces of a parametric partial differential equation. We denote the solution operator by
\[
    \mathcal{G}:\mathcal{A}\rightarrow \mathcal{U},\qquad u=\mathcal{G}(a),
\]
where $a\in\mathcal{A}$ is an input function, such as an initial condition or a coefficient field, and $u\in\mathcal{U}$ is the corresponding solution. Given training data $\{(a_i,u_i)\}_{i=1}^{N}$, our goal is to construct a surrogate $\mathcal{G}_{\theta}$ that accurately predicts $u$ for unseen inputs $a$.

In all experiments, prediction accuracy is measured by the mean relative $L^2$ error
\[
    {\rm Err}
    =
    \frac{1}{N_{\rm test}}
    \sum_{i=1}^{N_{\rm test}}
    \frac{\|\widehat u_i-u_i\|_2}{\|u_i\|_2}.
\]


\subsection*{PCA--RaNN framework}

PCA--RaNN approximates the solution operator through a latent-space decomposition
\[
    \mathcal{G}_{\theta}
    =
    G_{\mathcal{U}}\circ \Phi_{\theta}\circ F_{\mathcal{A}},
\]
where $F_{\mathcal A}:\mathcal A\to\mathbb R^{d_{\mathcal A}}$ maps the input function to a low-dimensional latent vector, $G_{\mathcal U}:\mathbb R^{d_{\mathcal U}}\to\mathcal U$ reconstructs the output field from its latent coordinates, and $\Phi_\theta:\mathbb R^{d_{\mathcal A}}\to\mathbb R^{d_{\mathcal U}}$ is the randomized latent map.

The input and output encoders are obtained by centered principal component analysis. For the input data, let $\bar a$ be the sample mean and let $\{\varphi_{\mathcal{A},j}\}_{j=1}^{d_{\mathcal{A}}}$ be the leading PCA modes. The input encoder is
\[
    F_{\mathcal{A}}(a)
    =
    \left(
    \langle a-\bar a,\varphi_{\mathcal{A},1}\rangle,\ldots,
    \langle a-\bar a,\varphi_{\mathcal{A},d_{\mathcal{A}}}\rangle
    \right)^{\top}.
\]
The output encoder $F_{\mathcal{U}}$ and decoder $G_{\mathcal{U}}$ are defined analogously from the output mean and output PCA modes. Thus the original training data are converted into latent pairs
\[
    x_i=F_{\mathcal{A}}(a_i)\in\mathbb{R}^{d_{\mathcal{A}}},
    \qquad
    y_i=F_{\mathcal{U}}(u_i)\in\mathbb{R}^{d_{\mathcal{U}}}.
\]


\subsection*{Randomized latent regressor}

The latent map $y\approx \Phi_{\theta}(x)$ is approximated by a two-hidden-layer randomized neural network with a skip connection. For a latent input $x$, we define
\[
    h^{(1)}(x)
    =
    \tanh\left(
    r_1\left(W^{(1)}x-c^{(1)}\right)
    \right),
\]
\[
    h^{(2)}(x)
    =
    \tanh\left(
    r_2\left(W^{(2)}h^{(1)}(x)-c^{(2)}\right)
    \right),
\]
and
\[
    \Phi_{\theta}(x)
    =
    {V^{(1)}}^{\top}h^{(1)}(x)
    +
    {V^{(2)}}^{\top}h^{(2)}(x).
\]
Here $W^{(1)}$ and $W^{(2)}$ are randomly sampled and fixed throughout training. The shift vectors $c^{(1)}$ and $c^{(2)}$ are obtained from empirical quantiles of the corresponding projected training data: for each neuron, the shift is set to a quantile level sampled uniformly from $[0.1,0.9]$, and kept fixed. Only the output weights $V^{(1)}$ and $V^{(2)}$ are trained.

Let
\[
    h(x)=
    \begin{bmatrix}
    h^{(1)}(x)\\
    h^{(2)}(x)
    \end{bmatrix},
    \qquad
    V=
    \begin{bmatrix}
    V^{(1)}\\
    V^{(2)}
    \end{bmatrix}.
\]
For fixed random features and fixed scales $(r_1,r_2)$, the model is linear in $V$:
\[
    \Phi_{\theta}(x)=h(x)^{\top}V.
\]
Given the hidden-feature matrix
\[
    H=
    \begin{bmatrix}
    h(x_1)^{\top}\\
    \vdots\\
    h(x_N)^{\top}
    \end{bmatrix},
    \qquad
    Y=
    \begin{bmatrix}
    y_1^{\top}\\
    \vdots\\
    y_N^{\top}
    \end{bmatrix},
\]
the output weights are obtained by least squares,
\[
    V^{\star}
    =
    \arg\min_V \|HV-Y\|_F^2.
\]
This closed-form training step replaces the non-convex optimization of a fully trained neural network.


\subsection*{Energy-matched scale selection}

The scales $r_1$ and $r_2$ determine the strength of the random projections before the nonlinear activation. We first compute them by a data-dependent energy-matching rule. Let
\[
    S_X=\frac{1}{N}\sum_{i=1}^{N}\|x_i\|_2^2,
    \qquad
    S_Y=\frac{1}{N}\sum_{i=1}^{N}\|y_i\|_2^2.
\]
We allocate fractions $\gamma_1$ and $\gamma_2$ of the output latent energy to the two hidden layers, with $\gamma_1+\gamma_2=1$.

The energy-matching rule provides a heuristic initialization. Since the entries of $W^{(1)}$ are sampled independently from $\mathrm{Unif}[-1,1]$, we have $\mathbb{E}[(W^{(1)}_{ij})^2]=1/3$. Ignoring the empirical anchor values and using the small-argument approximation $\tanh z\approx z$, the first-layer features satisfy $h^{(1)}(x)\approx r_1 W^{(1)}x$. Consequently,
\[
\mathbb{E}\|h^{(1)}(x)\|_2^2 \approx r_1^2 \frac{n_1}{3}\|x\|_2^2.
\]
Averaging over the training samples and matching the expected first-layer feature energy to the fraction $\gamma_1 S_Y$ of the output latent energy gives
\[
r_1^2 \frac{n_1}{3} S_X \approx \gamma_1 S_Y,
\]
which yields the first-layer energy-based scale
\[
    r_{1,{\rm var}}
    =
    \left(
    \frac{3\gamma_1S_Y}{n_1S_X}
    \right)^{1/2}.
\]

To avoid extremely small or saturated features, we clip it by
\[
    \bar r_1
    =
    \max\left\{
    r_{1,\min},
    \min\{r_{1,{\rm var}},r_{1,\max}\}
    \right\},
\]
where
\[
    r_{1,\min}
    =
    \left(
    \frac{3\tau^2}{n_1S_X}
    \right)^{1/2},
    \qquad
    r_{1,\max}
    =
    \left(
    \frac{3\delta^2}{S_X}
    \right)^{1/2}.
\]

After computing $\bar r_1$, we form the pilot first-layer features $H^{(1)}_{\rm pilot}$ and define
\[
    S_{H_1}=\frac{1}{N}\|H^{(1)}_{\rm pilot}\|_F^2.
\]
The second-layer scale is initialized by
\[
    r_{2,{\rm var}}
    =
    \left(
    \frac{3\gamma_2S_Y}{n_2S_{H_1}}
    \right)^{1/2},
\]
and clipped as
\[
    \bar r_2
    =
    \max\left\{
    r_{2,\min},
    \min\{r_{2,{\rm var}},r_{2,\max}\}
    \right\},
\]
with
\[
    r_{2,\min}
    =
    \left(
    \frac{3\tau^2}{n_2S_{H_1}}
    \right)^{1/2},
    \qquad
    r_{2,\max}
    =
    \left(
    \frac{3\delta^2}{S_{H_1}}
    \right)^{1/2}.
\]

We use $\gamma_1=0.7$ and $\gamma_2=0.3$ for all benchmark problems. The clipping parameters are fixed to $\tau=\delta=2$ throughout the experiments. The pair $(\bar r_1,\bar r_2)$ defines the energy-matched PCA--RaNN. Because this derivation ignores anchor shifts, activation saturation, feature correlations, and the subsequent least-squares readout, the energy-matched scales should be interpreted as a stable data-dependent initialization rather than an optimality condition. This version requires no iterative optimization of hidden-layer parameters and is therefore the fastest variant.


\subsection*{BFGS refinement of scale parameters}

Although the energy-matched scales provide a stable initialization, they do not directly minimize the prediction error. We therefore introduce a refined variant that optimizes only the two scale parameters while keeping all random weights fixed.

We parameterize
\[
    r_1(\eta_1)=\bar r_1\exp(\eta_1),
    \qquad
    r_2(\eta_2)=\bar r_2\exp(\eta_2),
\]
where $\eta=(\eta_1,\eta_2)^{\top}$. The exponential parameterization ensures positive scales.

The nominal training set is split into an inner training subset and a validation subset with an $80/20$ ratio. The validation subset is used only for selecting the scale parameters and is not used in the final test evaluation. In the BFGS-refined variant, the PCA encoders and decoders used during scale selection are constructed from the inner training subset only, so that the validation subset is not used for PCA fitting, readout fitting or scale initialization. After the scales are selected, this fixed PCA representation is kept unchanged, and the final output layer is refitted on the available training samples by least squares using the refined scales.
For each candidate $\eta$, the output weights are fitted on the inner training subset by ridge regression,
\[
    A_{\lambda}(\eta)
    =
    \left(
    H_{\rm tr}(\eta)^{\top}H_{\rm tr}(\eta)+\lambda I
    \right)^{-1}
    H_{\rm tr}(\eta)^{\top}Y_{\rm tr}.
\]
We use a fixed ridge parameter, with $\lambda=10^{-4}$ for Burgers, Darcy flow, and Navier--Stokes, and $\lambda=10^{-12}$ for the backward heat equation benchmark. The scale parameters are selected by minimizing the validation loss
\[
    J(\eta)
    =
    \frac{1}{N_{\rm val}d_{\mathcal{U}}}
    \left\|
    H_{\rm val}(\eta)A_{\lambda}(\eta)-Y_{\rm val}
    \right\|_F^2.
\]
We minimize $J(\eta)$ using BFGS initialized at $\eta=0$, using analytical gradients with respect to the two scale parameters. After obtaining $\eta^{\star}$, the refined scales are
\[
    r_1^{\star}=\bar r_1\exp(\eta_1^{\star}),
    \qquad
    r_2^{\star}=\bar r_2\exp(\eta_2^{\star}).
\]
The final output layer is then refitted on the full training set by least squares using the fixed refined scales. This gives the BFGS-refined PCA--RaNN.

For ensembles, BFGS is performed once on a representative random-feature realization, and the resulting scales are reused as shared global hyperparameters for all ensemble members. Each member then refits its own least-squares readout on the full training set. Thus, the shared scales are not member-wise validation optima.


\subsection*{Ensemble prediction}

To reduce the variance caused by random features, we use an ensemble of $K$ independently initialized PCA--RaNN models. The ensemble prediction is
\[
    \mathcal{G}_{\rm ens}(a)
    =
    \frac{1}{K}\sum_{k=1}^{K}\mathcal{G}^{(k)}_{\theta}(a).
\]
All ensemble members share the same PCA encoders and decoders, but use independently sampled random hidden weights. For the BFGS-refined version, the optimized scales are shared across ensemble members, and each member has its own least-squares output layer.


\subsection*{Uncertainty quantification}

The ensemble structure provides an empirical estimate of predictive variability. For a test input $a$, let $\widehat u^{(1)}(x_j),\ldots,\widehat u^{(K)}(x_j)$ be the predictions of the $K$ ensemble members at grid point $x_j$. The ensemble mean and standard deviation are
\[
    \mu(x_j)
    =
    \frac{1}{K}
    \sum_{k=1}^{K}
    \widehat u^{(k)}(x_j),
\]
and
\[
    \sigma(x_j)
    =
    \left[
    \frac{1}{K-1}
    \sum_{k=1}^{K}
    \left(
    \widehat u^{(k)}(x_j)-\mu(x_j)
    \right)^2
    \right]^{1/2}.
\]

To obtain calibrated prediction intervals, we use split conformal prediction. Given an independent calibration set $\mathcal{D}_{\rm cal}=\{(a_i,u_i)\}_{i=1}^{N_{\rm cal}}$, we compute the nonconformity scores
\[
    s_{ij}
    =
    \frac{|u_i(x_j)-\mu_i(x_j)|}{\max\{\sigma_i(x_j),\varepsilon\}},
    \qquad
    i=1,\ldots,N_{\rm cal},
    \quad
    j=1,\ldots,M,
\]
where $M$ is the number of grid points and \(\varepsilon>0\) is a small constant used only to avoid division by zero when the ensemble standard deviation is numerically vanishing. In all experiments, we set \(\varepsilon=10^{-8}\), which is several orders of magnitude smaller than the typical ensemble standard deviations observed on the calibration sets and therefore has negligible effect on the reported interval widths except at points with nearly zero ensemble variance. Let $n_{\rm cal}=N_{\rm cal}M$ and let $s_{(1)}\leq\cdots\leq s_{(n_{\rm cal})}$ denote the sorted pooled nonconformity scores. For a nominal level $\alpha$, we set
\[
    q_{1-\alpha}=s_{(k)},
    \qquad
    k=\left\lceil (n_{\rm cal}+1)(1-\alpha)\right\rceil,
\]
with the convention $q_{1-\alpha}=+\infty$ if $k>n_{\rm cal}$. The conformal prediction interval for a new input is
\[
    {\rm PI}_{1-\alpha}(a,x_j)
    =
    \left[
    \mu(x_j)-q_{1-\alpha}\max\{\sigma(x_j),\varepsilon\},
    \mu(x_j)+q_{1-\alpha}\max\{\sigma(x_j),\varepsilon\}
    \right].
\]

Under the standard exchangeability assumption on the calibration and test samples, this pooled construction provides finite-sample pooled marginal coverage at the prescribed confidence level. This is neither a pointwise guarantee at each fixed grid location, nor a simultaneous or conditional one.


\subsection*{Rapid online adaptation}

The linear output layer also allows fast online adaptation when a small number of new input--output pairs become available. Suppose the hidden-feature extractor has been fixed, and let $V\in\mathbb{R}^{(n_1+n_2)\times d_{\mathcal{U}}}$ be the current output weight matrix. For a new latent pair $(x_{\rm new},y_{\rm new})$, let $h_{\rm new}=h(x_{\rm new})\in\mathbb{R}^{n_1+n_2}$ be its hidden-feature vector. Recursive least squares updates the output layer without retraining the hidden layers. Let
\[
    P=(H^{\top}H+\lambda I)^{-1}
\]
be the inverse covariance matrix associated with the previous training data. In the online adaptation experiments, $\lambda$ is fixed to $1\times10^{-2}$, used only to initialize the inverse covariance matrix for the RLS update. The update is
\[
    k = \frac{P h_{\rm new}}{1+h_{\rm new}^{\top}P h_{\rm new}},
\]
\[
    e_{\rm new}^{\top} = y_{\rm new}^{\top} - h_{\rm new}^{\top}V,
\]
\[
    V \leftarrow V + k e_{\rm new}^{\top},
    \qquad
    P \leftarrow P - k h_{\rm new}^{\top}P .
\]
For multiple new samples, the update is applied sequentially. Since only the linear output layer is modified, the cost per new sample is $O((n_1+n_2)^2)$, which is substantially cheaper than retraining a fully optimized neural operator.


\subsection*{Benchmark problems and datasets}

We evaluate the methods on four parametric PDE benchmarks: the viscous Burgers' equation, Darcy flow, the two-dimensional incompressible Navier--Stokes equation in vorticity form, and the backward heat equation. For ensemble predictions, we use $K=20$ independently initialized members for all benchmarks.

\paragraph{Burgers' equation.}
We consider the one-dimensional viscous Burgers' equation
\[
    u_t+u u_x-\nu u_{xx}=0,
    \qquad
    (x,t)\in(0,1)\times(0,1],
\]

with periodic boundary conditions and viscosity $\nu=0.01$. The task is to learn the operator $\mathcal{G}:u_0(x)\mapsto u(x,t)$ from the initial condition to the full space--time solution. Following the data-generation protocol of Wang et al.~(\cite{Wang2021Learning}), the initial conditions are sampled from a Gaussian random field with covariance operator $625(-\Delta+25I)^{-4}$, where the Laplacian is equipped with periodic boundary conditions. The numerical solutions are represented on a uniform $101\times101$ space--time grid. We use $N_{\rm train}=1000$ training samples and $N_{\rm test}=100$ test samples. The latent dimensions are $d_{\mathcal{A}}=d_{\mathcal{U}}=100$, and the hidden widths are $(n_1,n_2)=(500,200)$.

\paragraph{Darcy flow.}
We consider the elliptic Darcy equation
\[
    -\nabla\cdot(a(x,y)\nabla u(x,y))=f(x,y),
    \qquad
    (x,y)\in(0,1)^2,
\]
with homogeneous Dirichlet boundary conditions and fixed source term $f\equiv1$. The task is to learn the operator $\mathcal{G}:a(x,y)\mapsto u(x,y)$ from the permeability field to the pressure solution. We use the Darcy dataset from the Fourier neural operator benchmark~(\cite{Li2020Fourier}). The permeability field is generated from a Gaussian random field $\mu\sim\mathcal{N}(0,(-\Delta+9I)^{-2})$ (zero Neumann Laplacian), thresholded to $a=12$ if $\mu\geq0$ and $a=3$ otherwise. Both the permeability and pressure fields are represented on a uniform $29\times29$ grid. We use $N_{\rm train}=1000$ training samples and $N_{\rm test}=200$ test samples. The latent dimensions are $d_{\mathcal{A}}=d_{\mathcal{U}}=150$, and the hidden widths are $(n_1,n_2)=(3500,1000)$.

\paragraph{Navier--Stokes equation.}
We consider the two-dimensional incompressible Navier--Stokes equation in vorticity form,
\[
    \partial_t\omega+v\cdot\nabla\omega = \nu\Delta\omega+f, \qquad
    \nabla\cdot v=0,
\]
with viscosity $\nu=10^{-3}$ on the periodic domain $(0,1)^2$, where the forcing term is fixed as $f(x,y)=0.1\sin(2\pi(x+y))+0.1\cos(2\pi(x+y))$, following the Fourier neural operator benchmark~(\cite{Li2020Fourier}). The task is to learn the temporal prediction operator $\mathcal{G}:\{\omega(\cdot,t_n)\}_{n=1}^{10}\mapsto\{\omega(\cdot,t_n)\}_{n=11}^{20}$, mapping the first ten vorticity snapshots to the subsequent ten. Initial vorticity fields are sampled from the Gaussian random field $\omega_0\sim\mathcal{N}(0,7^{3/2}(-\Delta+49I)^{-2.5})$ with periodic boundary conditions, and all fields are represented on a uniform $64\times64$ spatial grid. We use $N_{\rm train}=1000$ training samples and $N_{\rm test}=200$ test samples. The latent dimensions are $d_{\mathcal{A}}=d_{\mathcal{U}}=150$, and the hidden widths are $(n_1,n_2)=(3000,3000)$.

\paragraph{Backward heat equation.}
We consider the two-dimensional heat equation
\[
    u_t=\alpha(u_{xx}+u_{yy}),\qquad (x,y,t)\in(0,1)^2\times(0,1],
\]
with homogeneous Dirichlet boundary conditions and diffusivity $\alpha=0.05$. The task is to learn a data-driven backward reconstruction map $\mathcal{G}:u(\cdot,1)\mapsto\{u(\cdot,t_n)\}_{n=0}^{19}$, which maps the terminal temperature field to the preceding 20 snapshots at $t_n=0.05n$. This task is an ill-posed inverse problem in the classical sense: backward heat propagation amplifies high-frequency perturbations and does not define a stable inverse map on general function spaces. In this benchmark, the problem is therefore interpreted as a regularized statistical reconstruction task on the prescribed smooth data distribution; the finite-dimensional discretization, the smooth data prior and the PCA truncation provide implicit regularization. The initial condition is generated by sampling a Gaussian random field on the $64\times64$ grid with squared-exponential covariance $K(x,x')=\exp(-\|x-x'\|^2/(2l^2))$, $l=0.2$, multiplied by $\sin(\pi x)\sin(\pi y)$ to enforce the boundary condition. The forward heat equation is solved using a five-point finite-difference Laplacian and a Crank--Nicolson scheme with $\Delta t=0.005$, storing 21 equally spaced snapshots from $t=0$ to $1$. The final snapshot at $t=1$ serves as input, and the previous 20 as output. We use $N_{\rm train}=1000$, $N_{\rm test}=200$, latent dimensions $d_{\mathcal{A}}=d_{\mathcal{U}}=100$, and hidden widths $(n_1,n_2)=(400,100)$.

\paragraph{Training and reporting.}

For each benchmark, the PCA construction is performed before any validation, calibration or test evaluation and uses only the training snapshots. We compare two PCA--RaNN variants: the energy-matched version, in which the scales $(\bar r_1,\bar r_2)$ are determined by the energy-matching rule and the output layer is fitted by least squares; and the BFGS-refined version, in which the scales are initialized by energy matching and then refined on a validation loss. After the scale parameters have been selected, the output layer is refitted on the full training set by least squares. All reported PCA--RaNN results use a fixed train--test split and a fixed random seed unless otherwise stated. For the BFGS-refined variant, the training set is further split into an $80\%$ inner training subset and a $20\%$ validation subset for scale selection; after the scales are selected, the output layer is refitted on the full training set. We report the test relative $L^2$ error and the corresponding wall-clock training time measured on a single NVIDIA RTX A6000 GPU. The reported wall-clock time includes PCA fitting, random-feature construction, construction of shift vectors, all least-squares or ridge-regression solves, BFGS validation-loss and gradient evaluations when applicable, and fitting of all ensemble members. It excludes dataset generation and data loading. Single-model and ensemble results with $K=20$ members are presented separately. All models are implemented in PyTorch 2.5.1.


\clearpage
\section*{Acknowledgements}
This work was supported by the National Key R\&D Program of China (Grant No. 2025YFA1016400).

\section*{Author contributions}
Z.D. conceived the main technical idea, developed the PCA--RaNN framework, implemented the numerical experiments, processed the data, performed the visualizations, and wrote the initial draft. J.S. contributed to the BFGS scale-refinement strategy and implemented the BFGS-refined PCA--RaNN pipeline. D.M. led the broader research program supporting this study, provided resources and project-level support, offered strategic guidance, and reviewed and edited the manuscript. F.W. supervised Z.D., led the formulation of the research direction, coordinated the overall progress of the study, contributed key methodological ideas to the model design and implementation, provided research guidance and resources, and reviewed and edited the manuscript. All authors discussed the results and approved the final manuscript.

\section*{Competing interests}
The authors declare no competing interests.

\section*{Data availability}
The datasets used in this study are either generated following the procedures described in Methods or obtained from publicly available operator-learning benchmark datasets cited therein. The processed datasets and train/validation/test splits supporting the results of this study will be made available in a public repository before publication.

\section*{Code availability}
The code used to generate the results in this study, including PCA construction, randomized feature generation, scale refinement, ensemble prediction, conformal calibration and recursive least-squares online adaptation, will be made available to editors and reviewers during peer review and released publicly upon publication.

\section*{Additional information}
Supplementary information is available for this paper. Correspondence and requests for materials should be addressed to F.W.

\clearpage
\renewcommand{\figurename}{Extended Data Fig.}
\renewcommand{\tablename}{Extended Data Table}
\setcounter{figure}{0}  
\setcounter{table}{0}   

\begin{figure}[htbp]
    \centering
    \includegraphics[width=0.95\linewidth]{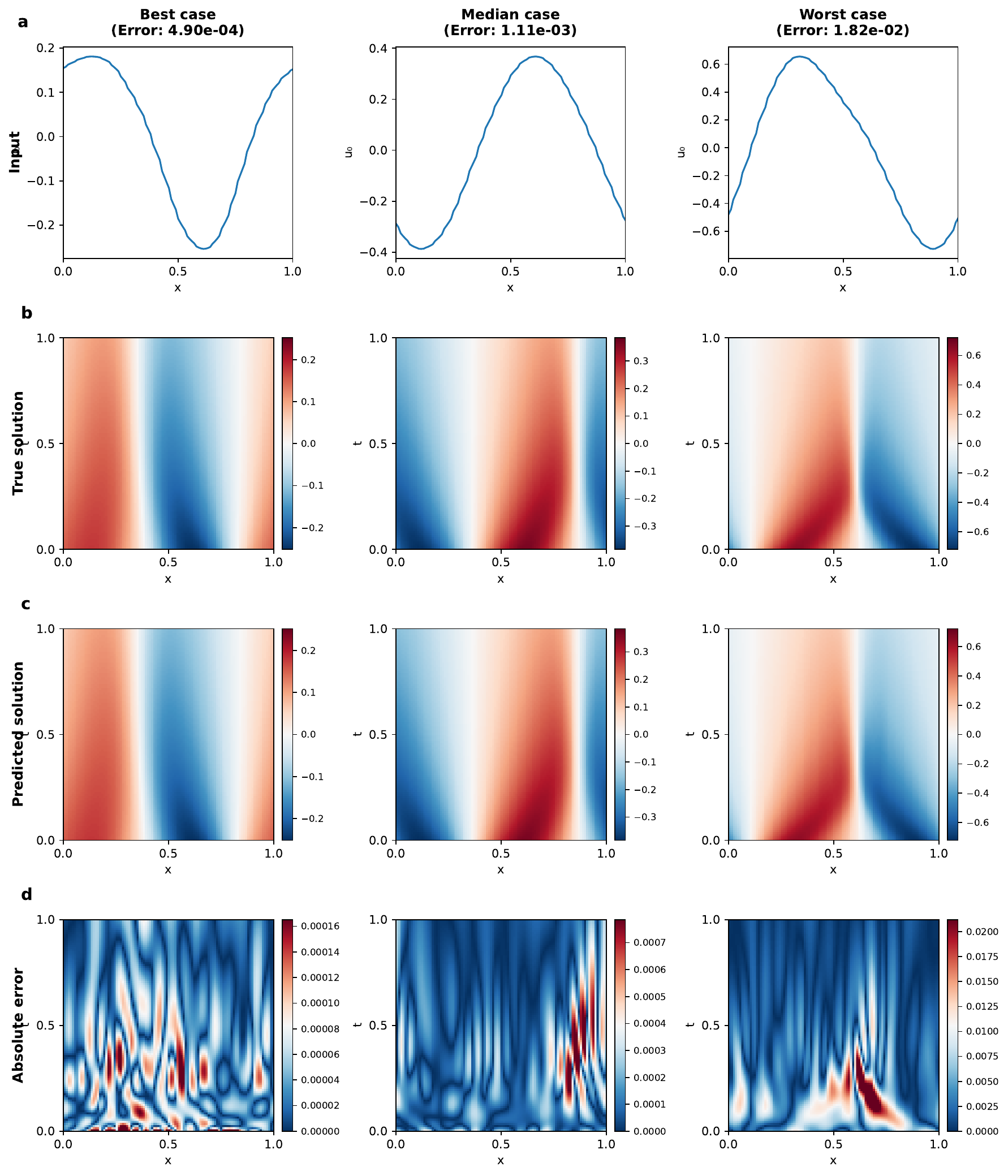}
    \caption{\textbf{Representative test samples for Burgers' equation.} 
    Three test cases are shown with increasing relative $L^2$ error from left to right (minimum, median, maximum). \textbf{a}, Input initial condition $u_0(x)$. \textbf{b}, True solution $u(x,t)$. \textbf{c}, Predicted solution. \textbf{d}, Absolute error. The model accurately captures the formation of steep gradients across all cases, with errors concentrated near sharp transition regions.}
    \label{fig:extended_Burgers}
\end{figure}

\begin{figure}[htbp]
    \centering
    \includegraphics[width=0.95\linewidth]{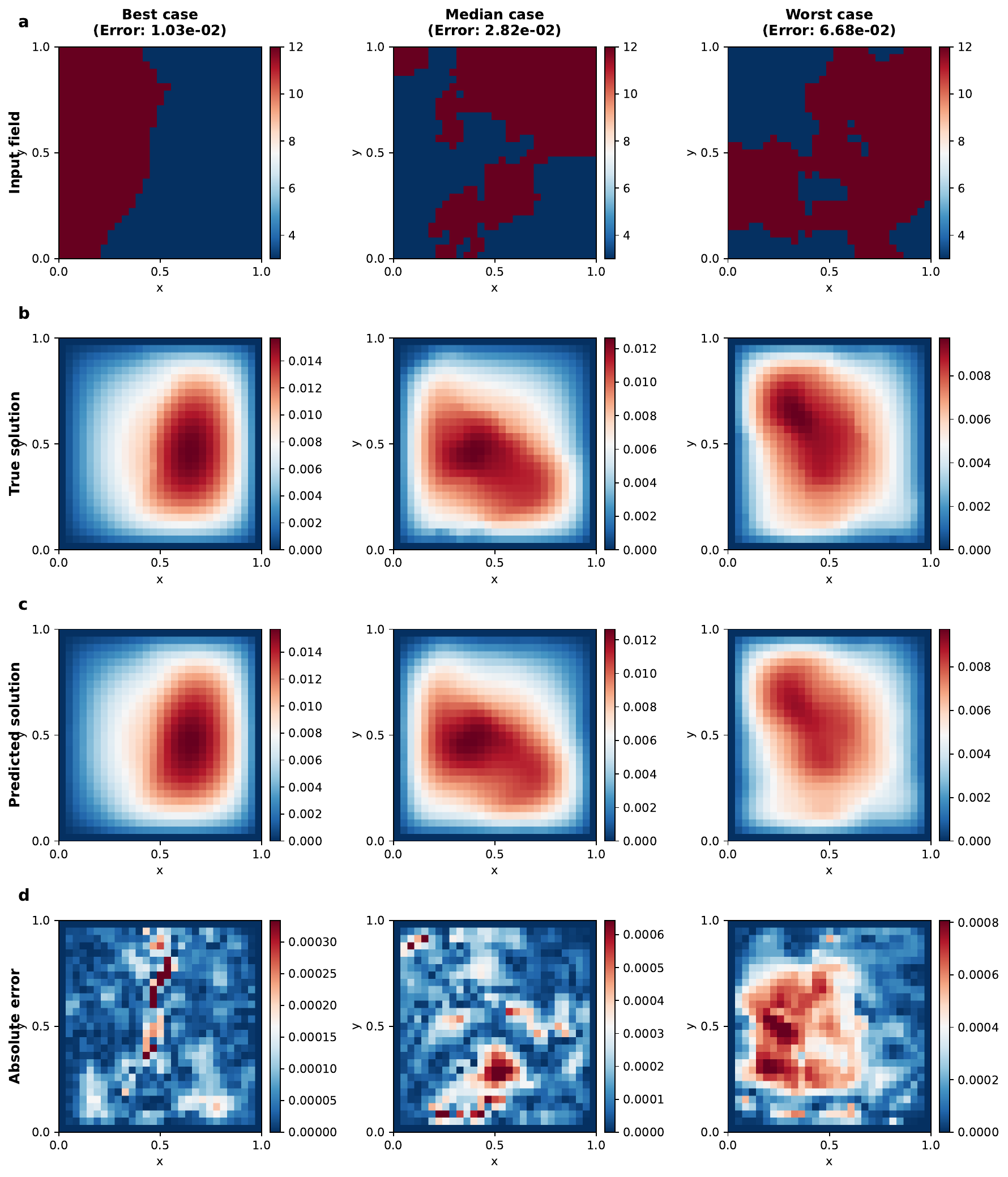}
    \caption{\textbf{Representative test samples for Darcy's equation.} 
    Three test cases are shown with increasing relative $L^2$ error from left to right (minimum, median, maximum). \textbf{a}, Input permeability field $a(x,y)$. \textbf{b}, True pressure field $u(x,y)$. \textbf{c}, Predicted pressure field. \textbf{d}, Absolute error. The model effectively captures the complex spatial patterns of the pressure field across heterogeneous media, with errors distributed primarily in high-gradient regions.}
    \label{fig:extended_Darcy}
\end{figure}

\begin{figure}[htbp]
    \centering
    \includegraphics[width=0.95\linewidth]{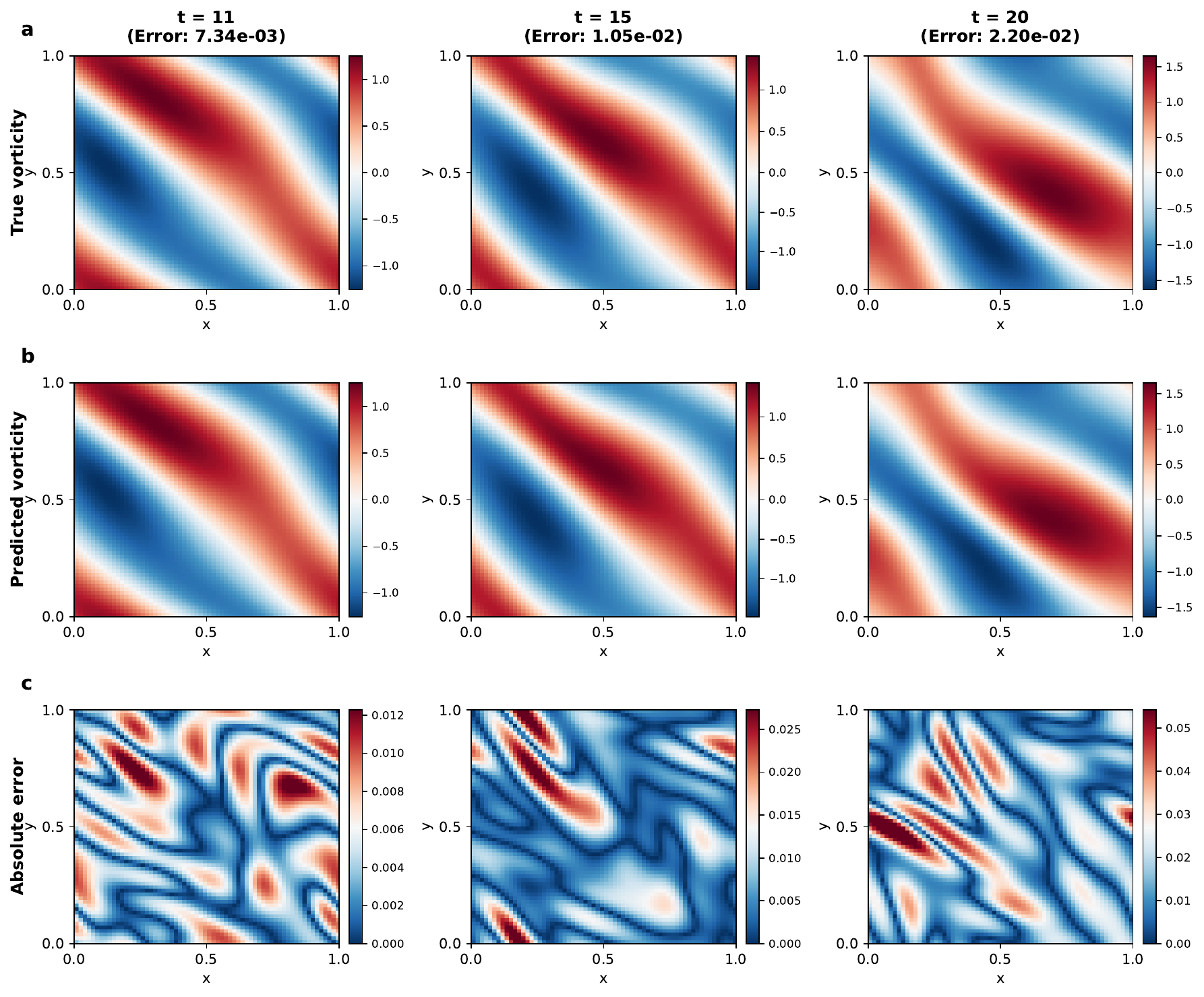}
    \caption{\textbf{Representative test sample for Navier--Stokes equation.} 
    Snapshots at $t=11$, $t=15$, and $t=20$ are shown from left to right. \textbf{a}, True vorticity field. \textbf{b}, Predicted vorticity field. \textbf{c}, Absolute error. The model captures the evolution of vorticity structures, with prediction errors remaining low throughout the predicted time window.}
    \label{fig:extended_NS}
\end{figure}

\begin{figure}[htbp]
    \centering
    \includegraphics[width=0.95\linewidth]{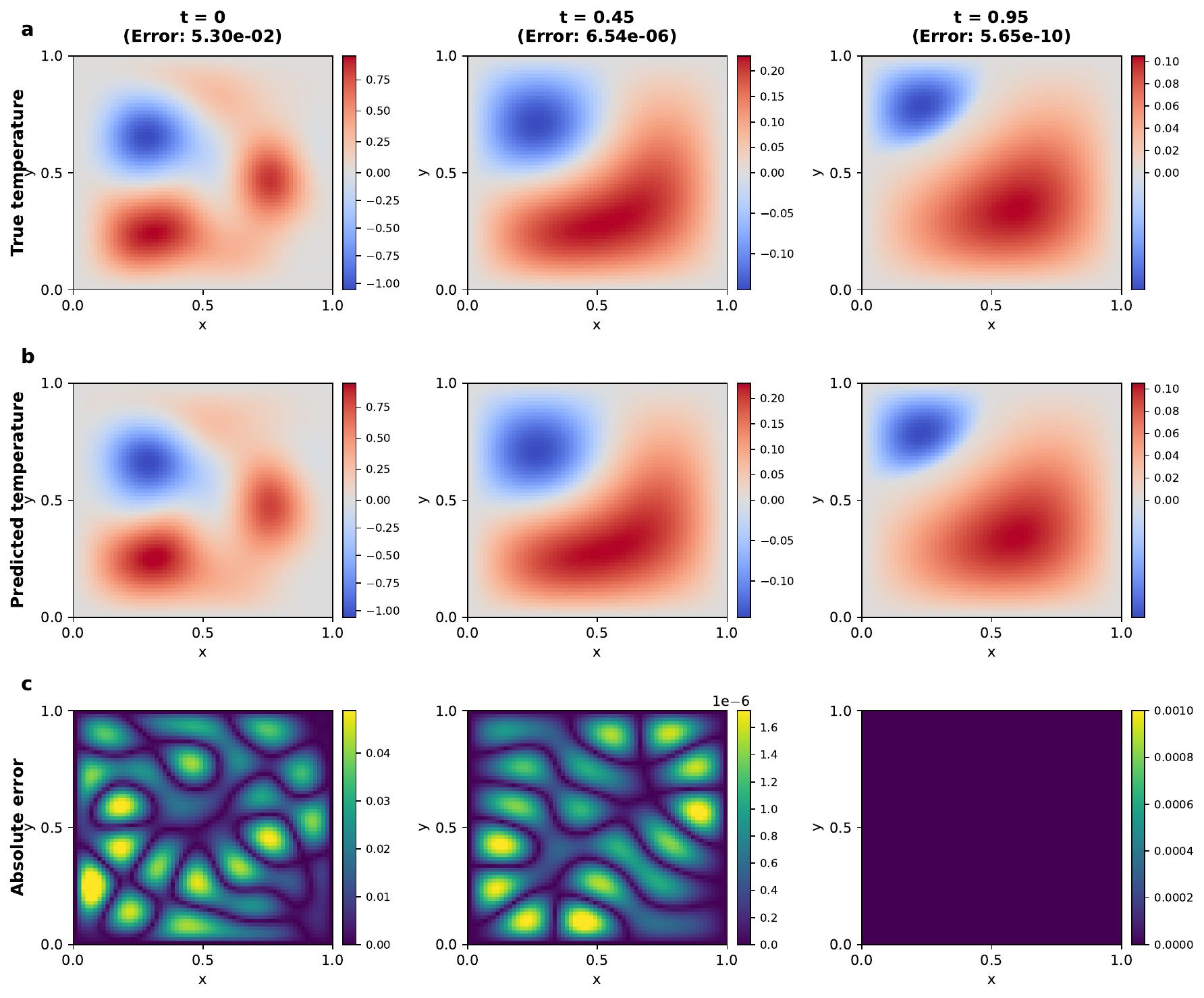}
    \caption{\textbf{Representative test sample for the backward heat equation.} 
    Snapshots at $t=0$, $t=0.45$, and $t=0.95$ are shown from left to right. 
    \textbf{a}, True temperature field. 
    \textbf{b}, Predicted temperature field. 
    \textbf{c}, Absolute error. 
    The model reconstructs the temperature trajectory for samples drawn from the same data prior, with errors increasing toward earlier times where the inverse problem is more ill-posed.}
    \label{fig:extended_Heat}
\end{figure}

\clearpage
\begin{table*}[ht]
\centering
\caption{\textbf{Comparison of single-member and ensemble PCA--RaNN models.}
Single-member results are averaged over $K=20$ independently initialized members; the reported standard deviation is the mean of the per-member test-sample standard deviations. Ensemble results use averaged predictions from the same $K=20$ members; the standard deviation is computed over test samples. For the BFGS variant, all members share the representative-member refined scales.}
\label{tab:ensemble_vs_single}
\begin{tabular}{@{}lcccccc@{}}
\toprule
& \multicolumn{3}{c}{\textbf{Single Member}} & \multicolumn{3}{c}{\textbf{Ensemble Model}} \\
\cmidrule(lr){2-4} \cmidrule(lr){5-7}
\textbf{Problem} & Rel. $L^2$ Error & Error Std & Time (s) & Rel. $L^2$ Error & Error Std & Time (s) \\
\midrule
\multicolumn{7}{c}{\textbf{EM}} \\
Burgers & $4.89\times10^{-3}$ & $5.09\times10^{-3}$ & 0.6 & $\mathbf{2.28\times10^{-3}}$ & $\mathbf{2.79\times10^{-3}}$ & 2.3 \\
Darcy   & $3.52\times10^{-2}$ & $1.17\times10^{-2}$ & 0.8 & $\mathbf{2.95\times10^{-2}}$ & $\mathbf{1.00\times10^{-2}}$ & 10.6 \\
Navier-Stokes & $3.12\times10^{-2}$ & $1.43\times10^{-2}$ & 3.6 & $\mathbf{2.59\times10^{-2}}$ & $\mathbf{1.34\times10^{-2}}$ & 19.4 \\
Backward heat eqn. & $3.86\times10^{-1}$ & $1.75\times10^{-1}$ & 3.1 & $\mathbf{2.89\times10^{-1}}$ & $\mathbf{9.05\times10^{-2}}$ & 4.6 \\
\midrule
\multicolumn{7}{c}{\textbf{BFGS}} \\
Burgers & $4.33\times10^{-3}$ & $4.52\times10^{-3}$ & 3.0 & $\mathbf{1.93\times10^{-3}}$ & $\mathbf{2.67\times10^{-3}}$ & 5.5 \\
Darcy   & $3.52\times10^{-2}$ & $1.12\times10^{-2}$ & 2.1 & $\mathbf{2.94\times10^{-2}}$ & $\mathbf{9.58\times10^{-3}}$ & 14.6 \\
Navier-Stokes & $2.93\times10^{-2}$ & $1.33\times10^{-2}$ & 67.0 & $\mathbf{2.47\times10^{-2}}$ & $\mathbf{1.26\times10^{-2}}$ & 88.7 \\
Backward heat eqn. & $6.27\times10^{-2}$ & $2.25\times10^{-2}$ & 7.7 & $\mathbf{5.59\times10^{-2}}$ & $\mathbf{1.71\times10^{-2}}$ & 9.8 \\
\bottomrule
\end{tabular}
\end{table*}

\clearpage
\appendix
\setcounter{section}{0}
\renewcommand{\thesection}{S\arabic{section}}
\section{Supplementary Note 1: Approximation theory for PCA--RaNN}
\label{sec:SI_theory}

This note provides theoretical support for the randomized-feature component of the PCA--RaNN framework introduced in the main text. We first recall the approximation guarantees for the PCA--NN method \cite{Bhattacharya2021}, and then extend them to the randomized setting, showing that a single-hidden-layer RaNN can replace the deterministic latent network while preserving the overall convergence rate with high probability. Throughout, we adopt the notation of the main paper: $\mathcal{A}$ and $\mathcal{U}$ are separable Hilbert spaces, $\mu$ is a probability measure on $\mathcal{A}$, and $\mathcal{G}:\mathcal{A}\to\mathcal{U}$ is the target operator.

\subsection{PCA--NN approximation (background)}
\label{sec:SI_PCANN}

The PCA--NN method \cite{Bhattacharya2021} approximates $\mathcal{G}$ by first compressing inputs and outputs via PCA and then learning a neural network $\Psi$ between the latent spaces. Its total error can be decomposed into a PCA projection error and a latent-network approximation error.

\begin{theorem}[PCA sampling error, \cite{Bhattacharya2021}]
\label{thm:SI_PCA_sampling}
Let $\mu$ be a probability measure on $\mathcal{A}$ with $\mathbb{E}_{a\sim\mu}\|a\|_{\mathcal{A}}^4<\infty$. For the empirical $d$-dimensional PCA subspace $V_{d,N}$ constructed from $N$ i.i.d. samples, define the empirical projection error $R_{\mathcal{A}}(N,d)=\mathbb{E}_{a\sim\mu}\|a-\Pi_{V_{d,N}}a\|_{\mathcal{A}}^2$ and the optimal projection error $R_{\mathcal{A}}(d)=\min_{\dim(V)=d}\mathbb{E}_{a\sim\mu}\|a-\Pi_V a\|_{\mathcal{A}}^2$. Then there exists a constant $Q=Q(\mu)>0$ such that
\[
\mathbb{E}\big[R_{\mathcal{A}}(N,d)\big] \le R_{\mathcal{A}}(d) + \sqrt{\frac{Qd}{N}}.
\]
\end{theorem}

\begin{theorem}[PCA--NN approximation, \cite{Bhattacharya2021}]
\label{thm:SI_PCANN}
Let $\mu$ be a probability measure on $\mathcal{A}$ with $\mathbb{E}_{a\sim\mu}\|a\|_{\mathcal{A}}^4<\infty$, and let $\mathcal{G}:\mathcal{A}\to\mathcal{U}$ be $\mu$-measurable and globally Lipschitz. For any $\varepsilon>0$, there exist latent dimensions $d_{\mathcal{A}},d_{\mathcal{U}}\in\mathbb{N}$, sample size $N\in\mathbb{N}$, a truncation parameter $\delta_0>0$, an approximation tolerance $\varepsilon_\Psi>0$, and a neural network $\Psi:\mathbb{R}^{d_{\mathcal{A}}}\to\mathbb{R}^{d_{\mathcal{U}}}$ such that
\[
\mathbb{E}_{a\sim\mu}\big\|G_{\mathcal{U}}\circ\Psi\circ F_{\mathcal{A}}(a)-\mathcal{G}(a)\big\|_{\mathcal{U}}^2
\le C\Big(\varepsilon_\Psi+\sqrt{\delta_0}+\sqrt{\frac{d_{\mathcal{A}}}{N}}+R_{\mathcal{A}}(d_{\mathcal{A}})+\sqrt{\frac{d_{\mathcal{U}}}{N}}+R_{\mathcal{U}}(d_{\mathcal{U}})\Big)<\varepsilon,
\]
where $C>0$ is a constant independent of $d_{\mathcal{A}},d_{\mathcal{U}},N,\varepsilon_\Psi,\delta_0$, and $F_{\mathcal{A}},G_{\mathcal{U}}$ are the PCA encoder and decoder. The truncation parameter $\delta_0$ controls the latent domain $D=[-M,M]^{d_{\mathcal{A}}}$ with $M=\sqrt{\mathbb{E}_{a\sim\mu}\|a\|_{\mathcal{A}}^2/\delta_0}$, and $\varepsilon_\Psi=\mathbb{E}_{a\sim\mu}\|\Psi(F_{\mathcal{A}}(a))-\varphi(F_{\mathcal{A}}(a))\|_2^2$ where $\varphi=F_{\mathcal{U}}\circ\mathcal{G}\circ G_{\mathcal{A}}$.
\end{theorem}

Theorem~\ref{thm:SI_PCANN} guarantees that, by suitable choices of dimensions and network complexity, the PCA--NN surrogate can achieve any prescribed accuracy. The proof in \cite{Bhattacharya2021} relies on the existence of a deterministic shallow network that approximates the latent map $\varphi$ uniformly on $D$.

\subsection{Randomized neural network approximation}
\label{sec:SI_RaNN}

We now show that a single-hidden-layer RaNN can approximate the same deterministic network $\Psi$ arbitrarily well with high probability, provided the random hidden parameters are drawn from a distribution with full support on a compact set containing the target parameters.

\begin{theorem}[Deterministic network existence, \cite{Barron1993}]
\label{thm:SI_Barron}
Let $\Omega\subset\mathbb R^{d}$ be bounded and let
$f:\Omega\to\mathbb R$ satisfy Barron's condition with constant $C_f$. Assume that $\rho$ is a bounded sigmoidal activation. For the normalized sigmoid case, one may take $0\le \rho\le 1$ with limits $0$ and $1$ at $\pm\infty$. Then for every $\varepsilon>0$ and every probability measure $\nu$ on $\Omega$, there exists a shallow network
\[
f_n(x)=\sum_{k=1}^n c_k\rho(a_k\cdot x+b_k)+c_0
\]
with bounded parameters such that $\|f-f_n\|_{L^2(\Omega,\nu)}^2\le \varepsilon$.
\end{theorem}

The implementation uses $\tanh$ activations. This is consistent with the above statement because the rescaled activation $(\tanh z+1)/2$ satisfies the normalized sigmoidal assumptions, and the affine rescaling can be absorbed into the output coefficients and the constant term. Hence the same approximation argument applies to the $\tanh$ activation used in PCA--RaNN.

For a vector-valued target $\varphi=(\varphi_1,\dots,\varphi_{d_\mathcal{U}})$ with each component satisfying Barron's condition, Theorem~\ref{thm:SI_Barron} yields component networks $\Psi_1,\dots,\Psi_{d_\mathcal{U}}$. Stacking their hidden layers and collecting their output coefficients into a vector-valued output layer gives a multi-output network $\Psi$ with $\mathbb{E}_{a\sim\mu}\|\Psi(F_{\mathcal{A}}(a))-\varphi(F_{\mathcal{A}}(a))\|_2^2\le\varepsilon_\Psi$, and the hidden-layer parameters of each component lie in a compact set $\mathcal{B}\subset\mathbb{R}^{d_\mathcal{A}}\times\mathbb{R}$.

\begin{lemma}[Parameter covering]
\label{lem:SI_covering}
Let $\{(w_k^*,b_k^*)\}_{k=1}^{N_\Psi}$ be the target parameters of $\Psi$, contained in a compact set $\mathcal{B}$. Let $\pi$ be a probability measure with full support on $\mathcal{B}$. For any $\eta>0$ and $\delta\in(0,1)$, define $p(\eta)=\min_k\pi(B_\eta(w_k^*,b_k^*))>0$. If $M_\Phi\ge\frac{1}{p(\eta)}\ln\frac{N_\Psi}{\delta}$ i.i.d. samples are drawn from $\pi$, then with probability at least $1-\delta$, every ball $B_\eta(w_k^*,b_k^*)$ contains at least one sample.
\end{lemma}
\begin{proof}
The proof follows a standard union bound and is omitted here for brevity.
\end{proof}

\begin{lemma}[Perturbation bound]
\label{lem:SI_perturb}
Let $\rho$ be $L_\rho$-Lipschitz on bounded sets. If $\|(w_k^*,b_k^*)-(w_k,b_k)\|_2<\eta$, then for any input $x$ with $\|(x,1)\|_2\le R_x$,
\[
|\rho({w_k^*}^\top x+b_k^*)-\rho(w_k^\top x+b_k)|\le L_\rho R_x\eta.
\]
Under the truncation $F_{\mathcal{A}}(a)\in D=[-M,M]^{d_\mathcal{A}}$, we have $\|F_{\mathcal{A}}(a)\|_2\le\sqrt{d_\mathcal{A}}M$, and hence
\[
\mathbb{E}_{a\sim\mu}|\rho({w_k^*}^\top F_{\mathcal{A}}(a)+b_k^*)-\rho(w_k^\top F_{\mathcal{A}}(a)+b_k)|^2\le L_\rho^2(d_\mathcal{A}M^2+1)\eta^2.
\]
\end{lemma}

Now we construct a vector-valued RaNN $\Phi:\mathbb R^{d_{\mathcal A}}\to\mathbb R^{d_{\mathcal U}}$ with $M_\Phi$ hidden neurons. On the covering event of Lemma~\ref{lem:SI_covering}, for each target neuron $k$, choose a sampled neuron $j(k)$ such that $(w_{j(k)},b_{j(k)})\in B_\eta(w_k^*,b_k^*)$. The map $k\mapsto j(k)$ is not required to be injective. For each sampled neuron $j$, define the vector-valued output coefficient
\[
    \widetilde c_j=\sum_{k:\,j(k)=j} c_k,
\]
with the convention that the sum is zero if no target neuron is assigned to $j$. If the deterministic network $\Psi$ contains an output bias term $c_0\in\mathbb R^{d_{\mathcal U}}$, we include the same bias term in $\Phi$. Thus,
\[
\Phi(x)=c_0+\sum_{j=1}^{M_\Phi}\widetilde c_j\rho(w_j^\top x+b_j)
=
c_0+\sum_{k=1}^{N_\Psi}c_k\rho(w_{j(k)}^\top x+b_{j(k)}).
\]

\begin{theorem}[RaNN high-probability approximation]
\label{thm:SI_RaNN}
Under the assumptions of Lemma~\ref{lem:SI_covering} and
Lemma~\ref{lem:SI_perturb}, let $M_\Phi$ be large enough that the covering event holds with probability at least $1-\delta$. Then, on this event,
\[
\mathbb{E}_{a\sim\mu}
\|\Phi(F_{\mathcal A}(a))-\Psi(F_{\mathcal A}(a))\|_2^2
\le
L_\rho^2(d_{\mathcal A}M^2+1)\eta^2
\left(\sum_{k=1}^{N_\Psi}\|c_k\|_2\right)^2 .
\]
In particular, if $\|c_k\|_2\le C_c$ for all $k$, then
\[
\mathbb{E}_{a\sim\mu}
\|\Phi(F_{\mathcal A}(a))-\Psi(F_{\mathcal A}(a))\|_2^2
\le
L_\rho^2(d_{\mathcal A}M^2+1)N_\Psi^2C_c^2\eta^2 .
\]
For any desired $\varepsilon_{\Phi\Psi}>0$, one can first choose $\eta>0$ sufficiently small so that the right-hand side is at most $\varepsilon_{\Phi\Psi}$, and then choose $M_\Phi$ sufficiently large so that the covering event holds with probability at least $1-\delta$.
\end{theorem}

\subsection{PCA--RaNN approximation theorem}
\label{sec:SI_PCARaNN}

Finally, we combine the PCA--NN error decomposition with the RaNN approximation result.

\begin{theorem}[PCA--RaNN approximation] \label{thm:SI_PCARaNN} 
Assume the hypotheses of Theorems~\ref{thm:SI_PCANN} and~\ref{thm:SI_RaNN}. Assume additionally that the output decoder $G_{\mathcal U}$ is Lipschitz continuous with constant $L_{\mathcal U}$, namely
\[ 
\|G_{\mathcal U}(y)-G_{\mathcal U}(y')\|_{\mathcal U} \le L_{\mathcal U}\|y-y'\|_2 . 
\] 
For the standard PCA decoder built from orthonormal PCA modes, $L_{\mathcal U}=1$; adding the empirical mean in centered PCA does not change this constant. Let $\varepsilon_\Psi$ be the latent-network approximation error from Theorem~\ref{thm:SI_PCANN}, and let $\varepsilon_{\Phi\Psi}$ be the RaNN approximation error from Theorem~\ref{thm:SI_RaNN}, achieved with probability at least $1-\delta$. Then for the composite surrogate $\mathcal G_\Phi=G_{\mathcal U}\circ\Phi\circ F_{\mathcal A}$, we have with probability at least $1-\delta$, 
\[ 
\mathbb{E}_{a\sim\mu} 
\|\mathcal G_\Phi(a)-\mathcal G(a)\|_{\mathcal U}^2 
\le 
2C\Big( 
\varepsilon_\Psi+\sqrt{\delta_0} 
+\sqrt{\frac{d_{\mathcal A}}{N}} 
+R_{\mathcal A}(d_{\mathcal A}) 
+\sqrt{\frac{d_{\mathcal U}}{N}} 
+R_{\mathcal U}(d_{\mathcal U}) 
\Big) 
+2L_{\mathcal U}^2\varepsilon_{\Phi\Psi}. 
\] 
Consequently, for any $\epsilon>0$, one can first choose $d_{\mathcal A}$ and $d_{\mathcal U}$ sufficiently large to make the PCA truncation errors small, then choose $N$ sufficiently large so that the empirical PCA sampling terms are small, choose $\delta_0$ and $\varepsilon_\Psi$ sufficiently small, and finally take $M_\Phi$ large enough so that the RaNN approximation error $\varepsilon_{\Phi\Psi}$ is small. With these choices, the total error is less than $\epsilon$ with probability at least $1-\delta$. 
\end{theorem}

\subsection{Remarks on centered PCA}
\label{sec:SI_remarks}

The theoretical analysis above is presented using uncentered PCA (i.e., projecting the raw data without subtracting the mean) for mathematical simplicity. However, in the numerical implementation described in the main text (Methods), we employ centered PCA, which subtracts the empirical mean $\hat\mu_N = \frac{1}{N}\sum_{j=1}^N a_j$ before constructing the covariance operator. Centered PCA is numerically more stable and yields zero-mean latent codes, which are often beneficial for subsequent neural network processing.

The difference between centered and uncentered PCA lies in the estimation of the mean. Let $\sigma = \mathbb{E}_{a\sim\mu}[a]$ be the population mean. For centered PCA, the optimal projection error becomes
\[
R^c_{\mathcal{A}}(d) = \min_{\dim(V)=d} \mathbb{E}_{a\sim\mu} \| a - \sigma - \Pi_V(a-\sigma) \|_{\mathcal{A}}^2,
\]
and the empirical version incurs an additional $O(1/N)$ term due to mean estimation. By a similar argument to that used in Theorem 3.4 of \cite{Bhattacharya2021}, one can show that
\[
\mathbb{E}\big[ R^c_{\mathcal{A}}(N,d) \big] \le R^c_{\mathcal{A}}(d) + \sqrt{\frac{Q' d}{N}} + \frac{C'}{N},
\]
with constants $Q',C'>0$ depending on the centered covariance and moment bounds. The extra $O(1/N)$ term decays faster than the dominant $O(1/\sqrt{N})$ terms, and therefore does not alter the qualitative convergence rate established in Theorems~\ref{thm:SI_PCANN} and~\ref{thm:SI_PCARaNN}. Hence, replacing uncentered PCA by centered PCA in practice does not affect the theoretical guarantee that the approximation error can be made arbitrarily small.

\subsection{Additional remarks}
\label{sec:SI_additional}

\begin{itemize}
    \item \textbf{Two-hidden-layer architecture.} The RaNN used in the main paper has two hidden layers with a skip connection. Extending the approximation theory to such multi-layer randomized architectures is an interesting direction for future work, but the single-layer analysis already demonstrates the principle that random features can replace deterministic networks without loss of expressivity.
    \item \textbf{Practical parameter sampling.} In our implementation, weights and biases are sampled from bounded uniform distributions. The full-support condition in Lemma~\ref{lem:SI_covering} is satisfied provided that the compact target parameter set $\mathcal B$ is contained in the support of these sampling distributions. Equivalently, the sampling ranges must be chosen large enough to cover the compact parameter set used in the approximation argument.
\end{itemize}

\section{Supplementary Note 2: Baseline configurations}
\label{sec:SI_baselines}

\renewcommand{\tablename}{Supplementary Table}
\setcounter{table}{0}

Supplementary Table~\ref{tab:baseline_config} summarizes the architectures and maximum training budgets used for the reimplemented PCA--NN, DeepONet, and FNO baselines.

\begin{table}[h]
\centering
\footnotesize
\setlength{\tabcolsep}{3pt}
\renewcommand{\arraystretch}{1.2}
\caption{\textbf{Baseline architectures and training budgets.}
The reimplemented baselines in Table~\ref{tab:benchmark} use the architectures and maximum training epochs listed below. DeepONet and FNO baselines are trained with Adam or AdamW, while PCA--NN baselines are trained with Nesterov SGD. Learning rates are decayed exponentially, by inverse-time scheduling, or by step-wise scheduling; full details are provided in the code repository.}
\label{tab:baseline_config}
\begin{tabular}{@{}lllc@{}}
\toprule
Problem & Method & Architecture \& hidden dims / basis size & Max epochs \\
\midrule
\multirow{3}{*}{Burgers} 
& PCA--NN
& \makecell[l]{PCA + MLP \\ $100\to500\to1000\to2000\to1000\to500\to100$}
& 30,000 \\
& DeepONet
& \makecell[l]{Branch/trunk FNN \\ Branch: 128 (3 layers); trunk: 128 (2 layers); $p=128$}
& 100,000 \\
& FNO
& \makecell[l]{2D FNO on $(t,x)$ grid \\ modes $(24,32)$, width 64, FC 128, 4 Fourier layers}
& 5,000 \\
\midrule
\multirow{3}{*}{Darcy} 
& PCA--NN
& \makecell[l]{PCA + MLP \\ $150\to500\to1000\to2000\to1000\to500\to150$}
& 20,000 \\
& DeepONet
& \makecell[l]{CNN branch + coordinate trunk \\ Branch: Conv(64)--Conv(128)--FC(128)--FC($p$); trunk: 128 (3 layers); $p=128$}
& 30,000 \\
& FNO
& \makecell[l]{2D FNO \\ modes $(12,12)$, width 32, FC 128, 4 Fourier layers}
& 2,000 \\
\midrule
\multirow{3}{*}{Navier--Stokes} 
& PCA--NN
& \makecell[l]{PCA + MLP \\ $150\to500\to1000\to2000\to1000\to500\to150$}
& 10,000 \\
& DeepONet
& \makecell[l]{CNN branch + periodic trunk \\ Branch: CNN(64,128,128,128) + FC(256); trunk: 128 (2 layers); $p=64$}
& 15,000 \\
& FNO
& \makecell[l]{2D multi-output FNO \\ modes $(12,12)$, width 32, FC 128, 4 Fourier layers}
& 2,000 \\
\midrule
\multirow{3}{*}{Backward heat eqn.} 
& PCA--NN
& \makecell[l]{PCA + MLP \\ $100\to500\to1000\to2000\to1000\to500\to100$}
& 20,000 \\
& DeepONet
& \makecell[l]{CNN branch + coordinate trunk \\ Branch: CNN(64,128,128,128) + FC(128); trunk: 128 (2 layers); $p=64$}
& 10,000 \\
& FNO
& \makecell[l]{2D multi-output FNO \\ modes $(16,16)$, width 64, FC 256, 4 Fourier layers}
& 2,000 \\
\bottomrule
\end{tabular}
\end{table}

\renewcommand{\refname}{Supplementary References}

\end{document}